\pdfoutput=1

\documentclass[11pt]{article}

\usepackage[review]{emnlp2023}
\usepackage{times}
\usepackage{mathtools}
\usepackage{latexsym}
\usepackage{algorithm}
\usepackage[noend]{algpseudocode}
\usepackage{mathrsfs}
\usepackage[T1]{fontenc}

\usepackage[utf8]{inputenc}
\usepackage{microtype}
\usepackage{amsthm}
\usepackage{graphicx}
\graphicspath{ {./images/} }
\usepackage{times}
\usepackage{latexsym}
\usepackage{amssymb}
\usepackage{amsmath}
\usepackage{bbm}
\usepackage{multirow}
\usepackage{inconsolata}
\usepackage{setspace}

\usepackage{xcolor, soul} 

\newtheorem{lemma}{Lemma}[section]
\newtheorem{hypothesis}{Hypothesis}[section]
\newtheorem{notation}{Notation}[section]
\newtheorem{remark}{Remark}[section]
\newtheorem{theorem}{Theorem}[section]
\newtheorem{definition}[theorem]{Definition}

\let\OLDthebibliography\thebibliography
\renewcommand\thebibliography[1]{
  \OLDthebibliography{#1}
  \setlength{\parskip}{5pt}
  \setlength{\itemsep}{0pt plus 2.0ex}
}

\title{TreePiece: Faster Semantic Parsing via Tree Tokenization}

\author{Sid Wang \\
  Meta Inc. USA \\
  \texttt{yuwang2020@meta.com} \\\And
    Akshat Shrivastava \\
   Meta Inc. USA \\
  \texttt{akshats@meta.com} \\
  \And
    Sasha Livshits \\
   Meta Inc. USA \\
  \texttt{alll@meta.com} \\
  }

\begin{document}
\nolinenumbers
\setlength{\abovedisplayskip}{4pt plus 2.95ex}
\setlength{\belowdisplayskip}{4pt plus 2.95ex}
{\makeatletter\acl@finalcopytrue
  \maketitle
}

\begin{abstract}
\emph{Autoregressive} (AR) encoder-decoder neural networks have proved successful in many NLP problems, including \emph{Semantic Parsing} -- a task that translates natural language to machine-readable \emph{parse trees}. However, the sequential prediction process of AR models can be slow. To accelerate AR for semantic parsing, we introduce a new technique called \emph{TreePiece} that tokenizes a parse tree into subtrees and generates one subtree per decoding step. On TOPv2 benchmark, TreePiece shows $4.6$ times faster decoding speed than standard AR, and comparable speed but significantly higher accuracy compared to \emph{Non-Autoregressive} (NAR). 
\end{abstract}

\section{Introduction}
\emph{Autoregressive} (AR) modeling \cite{Sutskever2014SequenceTS} is a commonly adopted framework in NLP where the next prediction is conditioned on the previously generated tokens. This paper focuses on AR approach for \emph{Semantic Parsing} \cite{Wong2005LearningFS}, an NLP task that converts a natural language utterance to a machine-interpretable symbolic representation  called \emph{logical form}. The sequence of actions to derive a logical form is isomorphic to a directed tree and often referred to as a \emph{parse tree} \cite{Zettlemoyer2005LearningTM}. 

The runtime latency of AR linearly correlates to the output length and could result in low inference speed \cite{Gu2017NonAutoregressiveNM,Wang2018SemiAutoregressiveNM}. \emph{Non-Autoregressive} (NAR) modeling \cite{Gu2017NonAutoregressiveNM,Wei2019ImitationLF,Ma2019FlowSeqNC}, on the other hand, is able to produce outputs in parallel and reduce latency by an order of magnitude \cite{Ghazvininejad2019ConstantTimeMT}. However, NAR performs considerably worse than its AR counterparts without extra training recipes \cite{Wang2019NonAutoregressiveMT,Zhou2020ImprovingNN,Su2021NonAutoregressiveTG}. The quality benefits of AR models therefore motivates us to improve their speed, rather than exploring NAR.\\ 
\\
\textbf{Our contributions} 
\smallskip


    $\bullet$ We propose a novel approach of tokenizing parse trees into large units called \emph{TreePiece units}, and then building an AR model that predicts one \emph{TreePiece unit} at a time, thus reducing the number of steps needed to generate a full parse tree. To the best of our knowledge, we are the first to extend subword-tokenizer algorithm to semantic trees such that each token is a subtree. 
    
    $\bullet$ We validate our approach on TOPv2 benchmark and show that \emph{TreePiece} decoding is 4.6 times faster than standard AR with less than $0.2\%$ accuracy degradation, and nearly as fast as NAR with up to $0.8\%$ accuracy gains.

    
    
    $\bullet$ We provide theoretical proofs to support our main algorithms and their variants.

\section{Methods}
\label{methods}
\subsection{Parse tree}
\label{parsetree}
In this paper, we utilize the \emph{hierarchical semantic representations} based on \emph{intent} and \emph{slot} \cite{Gupta2018SemanticPF}, allowing for modeling complex compositional queries in task-oriented dialog systems. See Figure \ref{glue} (LHS) for an example. We begin with a few recurring definitions. 

\begin{definition}[\emph{Ontology}]
\label{ontology}
\emph{A \emph{parse tree} node is called an \emph{ontology} iff it represents an \emph{intent/slot}, prefixed by \texttt{in:} and \texttt{sl:} respectively.}
\end{definition}
\begin{definition}[\emph{Skeleton}]
\label{skeleton}
\emph{The \emph{skeleton} of a \emph{parse tree} is the subtree that consists of all \emph{ontologies}.} 
\end{definition}
\begin{definition}[\emph{Utterance leaf}]
\label{leaf}
\emph{A \emph{text-span} node is called an \emph{utterance leaf} iff its parent is a \emph{slot}\footnote{We adopt the \emph{decoupled form} proposed in \cite{Aghajanyan2020ConversationalSP}, which simplifies compositional representations by ignoring text spans that are not \emph{utterance leaves}.}.}

\end{definition}



\subsection{TreePiece tokenizer algorithm}

We propose the algorithm to train a \emph{TreePiece tokenizer} that partitions a \emph{skeleton} into subtrees.
\begin{definition}[\emph{TreePiece vocabulary}]
\label{tpvocab}
\emph{The minimal open vocabulary of all possible subtree units returned from a \emph{TreePiece tokenizer} is called a \emph{TreePiece vocabulary}. We refer to each element in a \emph{TreePiece vocabulary} as a \emph{TreePiece unit}.}
\end{definition}
\begin{definition}[\emph{TreePiece simplex}]
\label{tpsimplex}
\emph{Let $\mathcal{V}$ be a \emph{TreePiece vocabulary} and $t$ be any \emph{TreePiece unit}. A \emph{TreePiece simplex} $\boldsymbol{p}$ is a mapping from $\mathcal{V}$ to the unit interval $[0, 1]$ such that $\sum_{t\in\mathcal{V}} \boldsymbol{p}(t) = 1$.}
\end{definition}
Our training proceeds in two stages: (1) generating \emph{TreePiece vocabulary} from a training corpus; (2) optimizing the \emph{TreePiece simplex} via an \emph{Expectation-Maximization} (EM) algorithm. 

\subsubsection{Vocabulary generation}
\label{vocab-gen}
This stage resembles the merging operation in \emph{Byte Pair Encoding} (BPE) \cite{Gage1994ANA,Sennrich2015NeuralMT}. Given a training corpus, denote its skeletons by $\mathscr{S}$. We initialize the \emph{TreePiece vocabulary} $\mathcal{V}$ as the set of \emph{ontologies} extracted from $\mathscr{S}$ and $\mathcal{F}_0$ as the map between \emph{ontologies} and their frequencies in $\mathscr{S}$. Now repeat the steps below until $\mathcal{V}$ reaches a pre-determined size:
\begin{itemize}
    \item Count the frequencies of all adjacent but unmerged \emph{TreePiece unit} pairs in $\mathscr{S}$. Find the most frequent pair $p^*$ and its frequency $n^*$. 
    \item Merge $p^*$ in every $S\in \mathscr{S}$ that contains $p^*$, add $p^*$ to $\mathcal{V}$, and update $\mathcal{F}_0$ with $\mathcal{F}_0(p^*) = n^*$.
\end{itemize}

\begin{figure*}[t]
     \centering
     \begin{minipage}[t]{1.0\linewidth}
         \centering
         \includegraphics[width=0.87\textwidth, height=9.5cm]{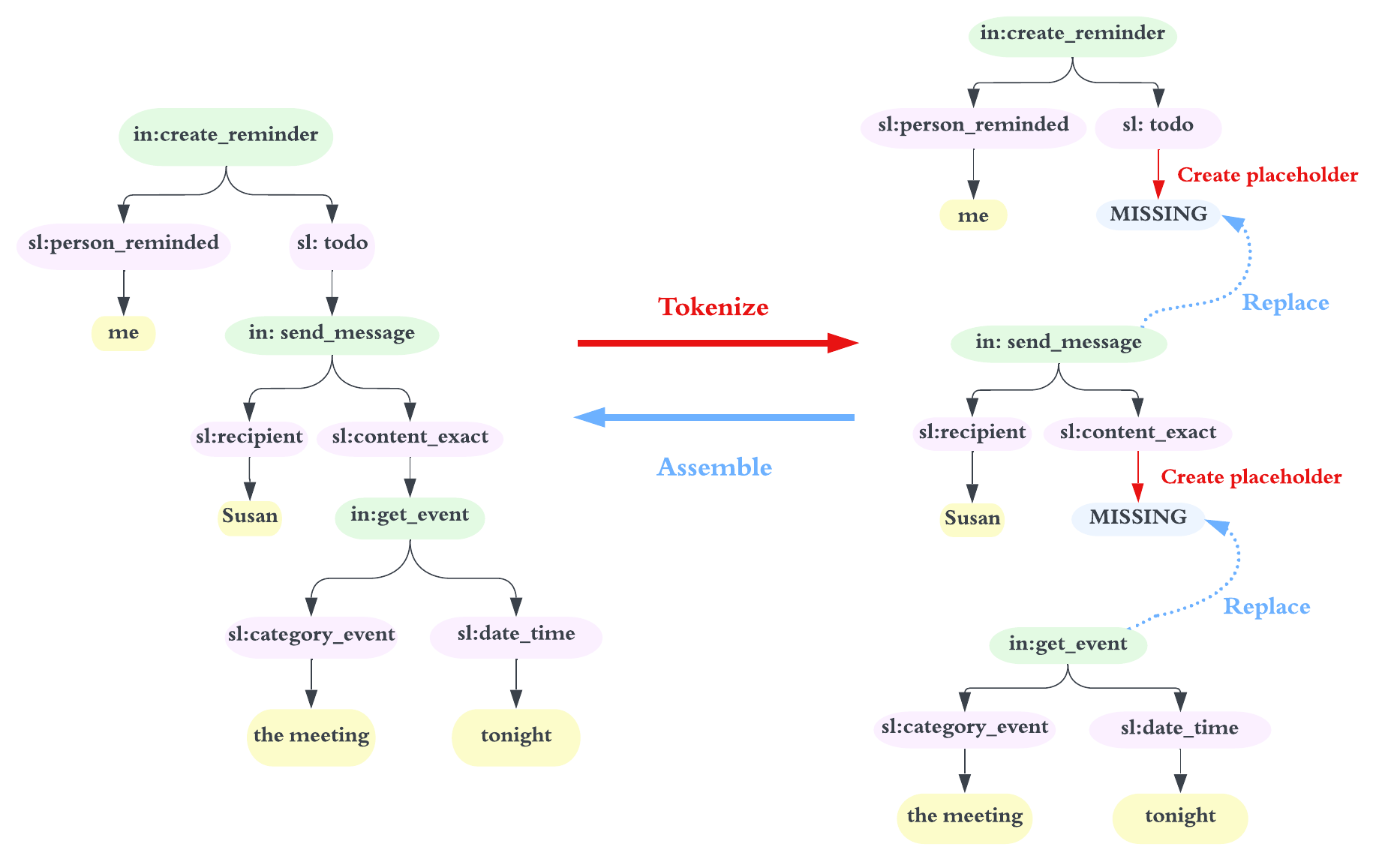}
     \end{minipage}
     \caption{Illustration of tokenizing parse tree/assembling TreePiece units with the placeholder design for given utterance ``\emph{Remind me to send Susan an email about the meeting tonight}''.}
     \label{glue}
\end{figure*}

\subsubsection{EM algorithm}
\label{em-section}
Briefly speaking, we initialize the \emph{TreePiece simplex} $\boldsymbol{p}_0$ by setting $\boldsymbol{p}_0(t)$ (for each $ t\in \mathcal{V}$) to be the normalized frequency ${\mathcal{F}_0(t)}/{\sum_{\tau\in \mathcal{V}}\mathcal{F}_0(\tau)}$, and will iteratively derive $\boldsymbol{{p}_{i+1}}$ for $i=0, 1, 2,\cdots$ according to the following inductive formula:
\begin{equation}
\label{em}
\boldsymbol{{p}_{i+1}} = \text{argmax}_{\boldsymbol{{p}}}\sum_{S\in\mathscr{S}}\mathbb{E}\big{[}\log \mathbb{P}(S, \pi; \boldsymbol{{p}}) \big{\rvert} S; \boldsymbol{{p}_i}\big{]}.
\end{equation}
Here $\pi$ denotes a partition of skeleton $S$. In general, problem \eqref{em} is NP-hard as it involves summing over $\Pi_S$, the set of all possible partitions of $S$:

\begin{eqnarray}
\begin{split}
\label{cond}
&\mathbb{E}\big{[}\log \mathbb{P}(S, \pi; \boldsymbol{{p}}) \big{\rvert} S; \boldsymbol{{p}_i}\big{]}\\
=& \sum_{\pi\in\Pi_S} \log \mathbb{P}(S, \pi; \boldsymbol{{p}}) \cdot \frac{\mathbb{P}(S, \pi; \boldsymbol{{p}_i})}{\mathbb{P}(S; \boldsymbol{{p}_i})}
\end{split}
\end{eqnarray}
To solve \eqref{em} in polynomial time, we impose the following assumption on the joint distribution: 
\begin{equation}
\label{singleton}
   \mathbb{P}(S, \pi; \boldsymbol{p}) \propto \left\{ 
  \begin{array}{ c l }
    \prod_{\tau\in\pi}\boldsymbol{p}(\tau) & \text{if }
    \pi=\pi_S(\boldsymbol{p})\\
    0                 & \textrm{otherwise,}
  \end{array}
\right.
\end{equation}
where $\pi_S(\boldsymbol{p}) = \text{argmax}_{\pi\in\Pi_{S}} {\prod}_{\tau \in \pi}\boldsymbol{p}(\tau).$ Applying \eqref{singleton}, we see that all but one summand in \eqref{cond} vanish. Algorithm \ref{token} outlines the EM process to solve \eqref{em}.
\begin{algorithm}
\caption{EM algorithm}\label{token}
\begin{algorithmic}
\State Choose $N_0 \in \mathbb{N}^+, \epsilon_0 >0$; initialize $i\gets 0$, $\Delta\gets +\infty$, $\mathcal{L}_{\text{prev}} \gets -\infty$. 

\While{$i < N_0$ and $\Delta > \epsilon_0$}
    \State $\mathcal{L}_{\text{curr}} \gets 0$, $\mathcal{F}^*\gets \text{Zero function on $\mathcal{V}$}$
    \For{$S\in \mathscr{S}$}
        \State Compute $\pi_S(\boldsymbol{p}_i)$ and $\mathbb{P}(S;\boldsymbol{p}_i)$ \Comment{E-step}
        \State $\mathcal{L}_{\text{curr}} \gets \mathcal{L}_{\text{curr}}+ \log\mathbb{P}(S;\boldsymbol{p}_i)$
        \For{$t\in \pi_S(\boldsymbol{p}_i)$}
            \State $\mathcal{F}^*(t) \gets \mathcal{F}^*(t) + 1$
        \EndFor
    \EndFor
    \For{$t \in \mathcal{V}$}
        \State $\boldsymbol{p}_{i+1}(t) \gets \frac{\mathcal{F}^*(t)}{\sum_{\tau\in\mathcal{V}}\mathcal{F}^*(\tau)}$ \Comment{M-step}
    \EndFor
    \State $i\gets i + 1, \Delta \gets \mathcal{L}_{\text{curr}} - \mathcal{L}_{\text{prev}}$
    \State $\mathcal{L}_{\text{prev}} \gets \mathcal{L}_{\text{curr}}$
\EndWhile
\end{algorithmic}
\end{algorithm}
In the \emph{E-step} of Algorithm \ref{token} we use a Viterbi-type algorithm \cite{Viterbi1967ErrorBF}, given in Appendix \ref{viterbi}. 

The following Theorem claims that Algorithm \ref{token} solves \eqref{em}. We defer its proof to Appendix \ref{proof}. 
\begin{theorem}
\label{confirmem}
Let $\boldsymbol{p}_0, \boldsymbol{p}_1, \boldsymbol{p}_2, \cdots$ be the sequence of TreePiece simplices obtained from Algorithm \ref{token}. If \eqref{singleton} holds, then $\boldsymbol{p}_{i+1}$ solves the optimization problem \eqref{em}, for $i\ge 0$. Especially, $\sum_{S\in\mathscr{S}}\log\mathbb{P}(S;\boldsymbol{p}_i)$ is monotonically non-decreasing in $i$.
\end{theorem}

\subsection{Modeling}
We describe the model that generates \emph{parse tree} components and the method to piece them together. 
\subsubsection{Modeling mechanism}
\label{model}
As illustrated in Figure \ref{tp_model}, an encoder computes the hidden states of a given utterance, then an AR decoder consumes the encoder hidden states and generate \emph{TreePiece units} autoregressively. The technique in Subsection \ref{recover} will allow us to put these units together and obtain a full \emph{skeleton}. The \emph{skeleton} then uniquely determines the numbers and positions of all utterance leaves (see Figure \ref{glue}), which offers us the convenience to use an NAR decoder to generate all utterance leaves within one step. 

\begin{figure}[h]
\centering
\includegraphics[width=0.22\textwidth, height=5.8cm]{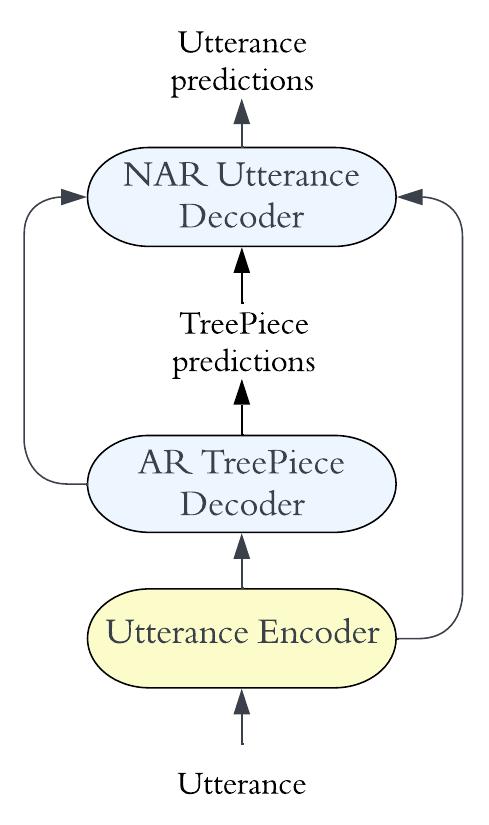}
\caption{TreePiece-based parse tree modeling design.}
\label{tp_model}
\end{figure}

\subsubsection{Assemble TreePiece units}
\label{recover}
Unlike subword-tokenization, where original sentence can be trivially recovered from subword units via string concatenation, there is no canonical way to reassemble \emph{TreePiece units}. To overcome this issue, we allow \emph{TreePiece units} to have placeholders\footnote{Finding all possible placeholder patterns is NP-hard and unnecessary. In Appendix \ref{oov} we provide a practical solution.}, and require that two units can only be joined at a placeholder node. This design provides a unique way to glue a sequence of ordered (e.g. pre/level-ordered) \emph{TreePiece units}, as shown in Figure \ref{glue}.

\section{Experiments}
\label{exp}
\subsection{Datasets}
\label{data}
We train, validate, and test our approach on the publicly available benchmark TOPv2 \cite{Chen2020LowResourceDA}, a multi-domain task-oriented semantic parsing dataset. The dataset provides a training/validation/test split. Throughout our experiments, we use the training split to train the models, the validation split for earlystopping, model checkpointing, and hyperparameter tuning, and the test split to report the best model's performance.
\subsection{Metrics}
\label{metrics}
We evaluate the model performance on two metrics: \emph{Exact Match} (EM) respectively \emph{Exact Match of Skeleton} (EM-S), defined to be the percentage of utterances whose \emph{logical forms} respectively \emph{skeletons} are correctly predicted \cite{Shrivastava2022RetrieveandFillFS}.  
\subsection{Baselines}
\label{baselines}
We compare our approach against 2 baselines: AR and NAR. Both baselines are sequence-to-sequence (seq2seq) that produces subword units of serialized logical forms. Their output space consists of \emph{ontologies} (prefixed by left bracket ``[''), \emph{utterance leaves}\footnote{We represent \emph{utterance leaves} in \emph{span-pointers} \cite{Shrivastava2021SpanPN} form to simplify the parsing task.}, and right bracket ``]''.
\smallskip
\\
\textbf{AR baseline} admits a standard AR structure. It has an \emph{autoregressive} decoder that generates serialized logical forms by producing one token at a time. 
\smallskip
\\
\textbf{NAR baseline} adopts \emph{mask-predict} \cite{Ghazvininejad2019ConstantTimeMT} with beam size $1$, which predicts the output length first and then generates all tokens in one step using a \emph{non-autoregressive} decoder.

\section{Results}
\label{results}

\begin{table}[]\centering
\begin{tabular}{cccc}
\hline
                                                               & TreePiece & AR    & NAR   \\ \hline
EM (\%)                                                       & 86.78     & \textbf{86.99} & 86.22 \\ \hline
EM-S(\%)                                                       & \textbf{89.14}     & 89.13 & 88.44 \\ \hline
\begin{tabular}[c]{@{}c@{}}Decoding (ms)\end{tabular}  & 33        & 152   & \textbf{27}    \\ \hline
\begin{tabular}[c]{@{}c@{}}Inference (ms)\end{tabular} & 74        & 196   & \textbf{68}    \\ \hline
\end{tabular}
\caption{Quality and latency of all models on TOPv2}
\label{key-results}
\end{table}



We train each model with $3$ random seeds, and report the averaged EM/EM-S test scores. To obtain latency, we infer the trained models on the test split of TOPv2 dataset with CPU and report the averaged milliseconds over all samples. We defer the model configurations, training details, and comparisons with prior work to Appendix \ref{conf-hps}. 

\subsection{Quality}
As shown in Table \ref{key-results}, TreePiece model sees $0.7\%$ relative improvements over NAR and $0.2\%$ relative degradation from AR in terms of EM, while achieving the best EM-S score among all approaches, especially showing $0.8\%$ relative improvement over NAR. We attribute TreePiece's high quality on skeleton predictions to its ability to respect the tree structure of logical forms and generating $100\%$ valid outputs by design so that the model can better focus on utterance-understanding without being distracted by any structure issue.

\begin{figure}[h]
\centering
\includegraphics[width=0.45\textwidth, height=5.8cm]{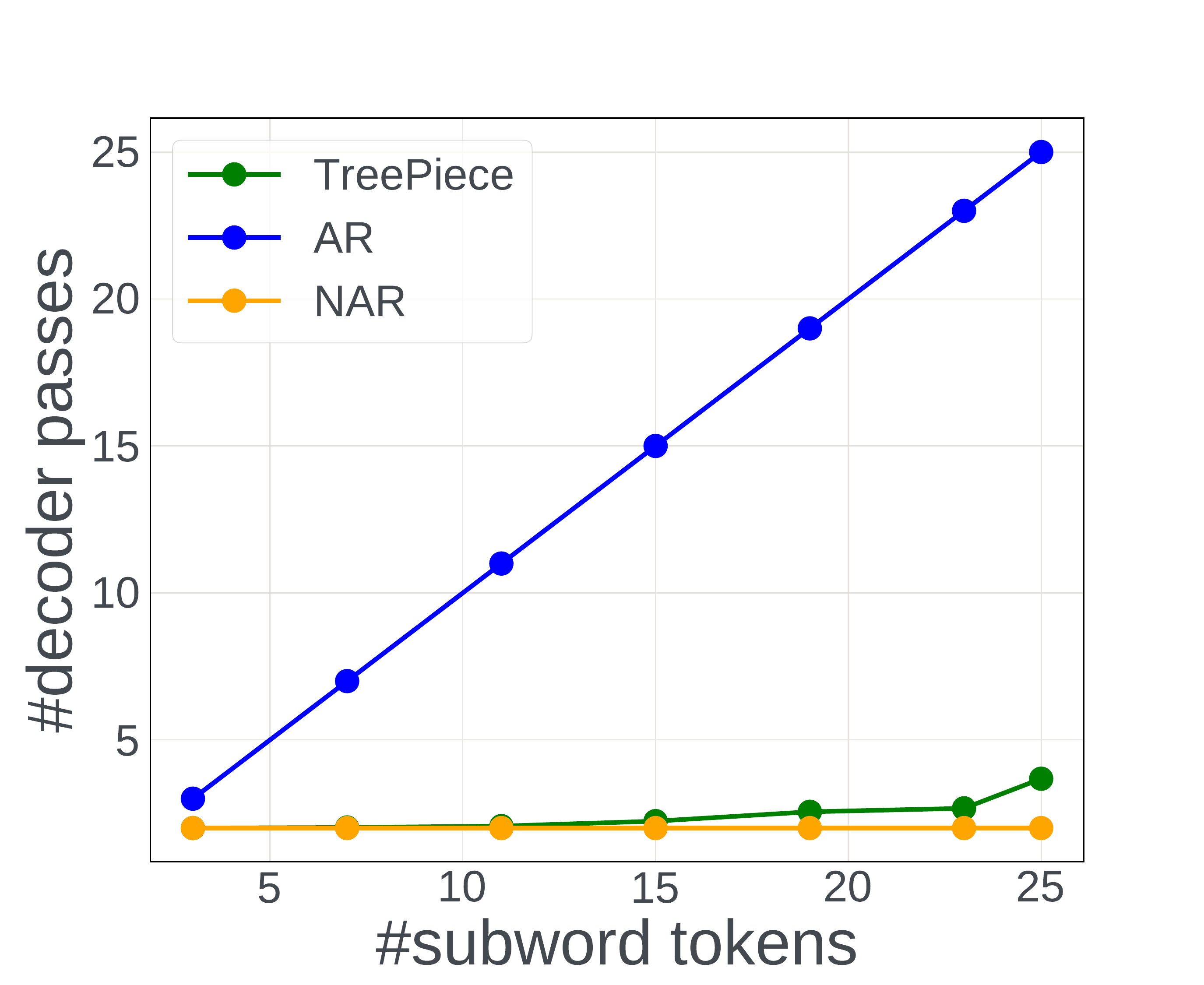}
\caption{Plot of averaged number of passes through decoder layers against number of tokens of a logical form tokenized by the BPE tokenizer of RoBERTa.}
\label{dec_steps}
\end{figure}

\subsection{Latency}
TreePiece makes decoding $4.6$ times faster and overall inference $2.7$ times faster than AR, with less than $10\%$ latency growth compared to NAR (see Table \ref{key-results}). The acceleration effects are illustrated in \ref{dec_steps}, showing that TreePiece substantially reduces the number of decoding steps and thereby requires fewer passes through decoder layers.\\
\\
\textbf{Related work} \emph{Autoregressive} modeling have been used in a range of \emph{Semantic Parsing} works \cite{Tai2015ImprovedSR,Cheng2017LearningSN,Dong2018CoarsetoFineDF}. Especially, the Sequence-to-Tree scheme was adopted by \cite{Dong2016LanguageTL}. To speed up the inference time, \emph{Non-autoregressive} modeling were introduced to the field of \emph{Machine Translation} \cite{Gu2017NonAutoregressiveNM,Lee2018DeterministicNN,libovicky-helcl-2018-end}, and later become popular in \emph{Semantic Parsing} as well \cite{Ghazvininejad2019ConstantTimeMT,Babu2021NonAutoregressiveSP,Shrivastava2021SpanPN}. However, to match the quality of AR, extra training stages are necessary such as \emph{Knowledge Distillation} from AR models \cite{Gu2017NonAutoregressiveNM,Lee2018DeterministicNN,Wei2019ImitationLF,Stern2019InsertionTF}. On the other hand, \cite{Rubin2020SmBoPSB} improves AR decoding's efficiency via \emph{Bottom-Up Parsing} \cite{Cheng2017LearningAE}. This paper takes a completely different path from all previous work by extending the subword tokenization algorithms \cite{nagata-1994-stochastic,scott2002,Sennrich2015NeuralMT,Kudo2018SubwordRI,Kudo2018SentencePieceAS} to trees.

\section*{Conclusion}
This paper proposes a novel technique to speed up \emph{Autoregressive} modeling for \emph{Semantic Parsing} via tokenizing parse trees into subtrees. We provide thorough elucidations and theoretical supports for our technique, and demonstrate significant improvements in terms of speed and quality over common baselines on the TOPv2 benchmark.


\bibliography{anthology,custom}
\bibliographystyle{acl_natbib}

\appendix
\section{Appendix}
\label{viterbi}

Algorithm \ref{viterbi-algo} outlines the Viterbi-type forward-backward \cite{Nagata1994ASJ} algorithm used in the EM procedure (ref. Algorithm \ref{em}) to compute the optimal tokenization and its probability for given skeleton $S$. We first obtain the set of all subtrees of $S$ that share the same root as $S$ denoted by $\mathcal{T}$. For convenience, we assume that $\mathcal{T}$ is graded by depth, such that $\mathcal{T}_d\subseteq \mathcal{T}$ denotes the set of all depth-$d$-subtrees. Next, we call \texttt{GetInitLogProbs} to initialize the log probability function on $\mathcal{T}$ as follows: 
\[\mathscr{L}(\mathfrak{t})=\begin{cases} 
      \log\boldsymbol{p}(\mathfrak{t}) &\text{ if }\mathfrak{t}\in \mathcal{V}, \\
      -\infty &\text{ otherwise}.
   \end{cases}\]
where $\mathcal{V}$ is the \emph{TreePiece vocabulary} and $\boldsymbol{p}$ the \emph{TreePiece simplex}. 

\begin{algorithm}
\caption{Forward-backward algorithm}\label{viterbi-algo}
\begin{algorithmic}
\Require TreePiece vocabulary $\mathcal{V}$, TreePiece simplex $\boldsymbol{p}$, and skeleton $S$.
\Ensure Parition $\pi_S(\boldsymbol{p})$ and probability $\mathbb{P}(S; \boldsymbol{p})$.
\State $\mathcal{T} \gets$ All subtrees of $S$ with the same root. 
\State $\mathscr{L} \gets$ \texttt{GetInitLogProbs}$(\boldsymbol{p})$
\State $\mathscr{P} \gets $ Constant map from $\mathcal{T}$ to \texttt{BOS} token
\State $d_{\text{max}}\gets\text{Depth of {S}}$
\State \textcolor{red}{// Forward begins}
\For{$d=1,2,\cdots,d_{\text{max}}$}
    \For{$\mathfrak{t}\in \mathcal{T}_d$}
        \For{$d^\prime=1,2,\cdots,d$}
            \For{$\mathfrak{t}^\prime\in \texttt{Filter}(\mathcal{T}_{d^\prime}, \mathfrak{t})$}
                \State $\Delta^*\gets$ $\mathfrak{t}^\prime\Delta\mathfrak{t}$ 
            
            \State $L^*\gets\mathscr{L}(\mathfrak{t}^\prime) + \sum_{\tau\in \Delta^*}\log\boldsymbol{p}(\tau)$
            \If{$L^* > \mathscr{L}(\mathfrak{t})$}
                \State $\mathscr{L}(\mathfrak{t}) \gets L^*$, $\mathscr{P}(\mathfrak{t}) \gets \mathfrak{t}^\prime$
            \EndIf
            \EndFor
        \EndFor
    \EndFor
\EndFor
\State $\mathbb{P}(S; \boldsymbol{p})\gets \exp(\mathscr{L}(S))$
\State \textcolor{red}{// Forward ends}
\State $\mathfrak{t}_{\text{curr}}\gets S$, $\pi_S(\boldsymbol{p})\gets \emptyset$
\State \textcolor{red}{// Backward begins}
\While{$\mathfrak{t}_{\text{curr}}\neq \texttt{BOS}$ token,}
    \State $\mathfrak{t}_{\text{prev}}\gets \mathscr{P}(\mathfrak{t}_{\text{curr}})$
    \State $\Delta^*\gets\mathfrak{t}_{\text{prev}}\Delta\mathfrak{t}_{\text{curr}}$
    \State $\pi_S(\boldsymbol{p}) \gets \pi_S(\boldsymbol{p})\bigcup \Delta^*$
    \State $\mathfrak{t}_{\text{curr}}\gets \mathfrak{t}_{\text{prev}}$
\EndWhile
\State $\pi_S(\boldsymbol{p}) \gets \pi_S(\boldsymbol{p})\bigcup \{\mathfrak{t}_{\text{curr}}\}$
\State \textcolor{red}{// Backward ends}
\State \Return $\pi_S(\boldsymbol{p}), \mathbb{P}(S; \boldsymbol{p})$
\end{algorithmic}
\end{algorithm}

The \emph{forward} step uses dynamic programming inductive on tree-depth to update all subtrees' log-probabilities and eventually obtain $\mathbb{P}(S;\boldsymbol{p})$ -- the probability of the skeleton $S$. The \emph{forward} step also returns a map $\mathscr{P}$ that stores for each $\mathfrak{t}\in\mathcal{T}$ the optimal position of its previous partition. Then in the \emph{backward} step we can backtrack along the path $S$, $\mathscr{P}(S), \mathscr{P}(\mathscr{P}(S)), \cdots$ to recover the optimal partition $\pi_S(\boldsymbol{p})$. Note that for dynamic programming's efficiency we apply a \texttt{Filter}$(\cdot, \mathfrak{t})$ method to narrow down the subtrees to those $\mathfrak{t}^\prime$ such that (1) $\mathfrak{t}^\prime$ is a subtree of $\mathfrak{t}$, (2) the set difference $\mathfrak{t}^\prime\Delta\mathfrak{t}$ has exactly one connected component and it is a \emph{TreePiece unit}. 

\section{Appendix}
\label{proof}
For convenience we adopt the following notations. 
\begin{notation}
\emph{Let $\boldsymbol{p}$ be a \emph{TreePiece simplex} and $\pi\eqqcolon [\tau_1,\cdots,\tau_k]$ be a partition where each $\tau_i$ is a \emph{TreePiece unit}. Define $\boldsymbol{p}(\pi)\eqqcolon\prod_{\tau\in \pi}\boldsymbol{p}(\tau)$.}
\end{notation}
\begin{notation}
\emph{Let $\pi\eqqcolon [\tau_1,\cdots,\tau_k]$ be a partition and $\tau$ be any \emph{TreePiece unit}. Define $n(\pi, \tau) \eqqcolon \sum_{\tau_i\in\pi}\mathbbm{1}_{\tau = \tau_i}$. In other words, $n(\pi, \tau)$ is the number of appearances of $\tau$ in $\pi$.}
\end{notation}
Now we introduce a general hypothesis and will prove a key lemma under this hypothesis.
\begin{hypothesis}
\label{general-hypo}
The joint distribution of skeleton $S$ and partition $\pi$ satisfies the following rule,
\begin{equation}
\label{general-hypo-eqn}
   \mathbb{P}(S, \pi; \boldsymbol{p}) \propto \left\{ 
  \begin{array}{ l l }
    \prod_{\tau\in \pi} \boldsymbol{p}(\tau)\cdot\chi(\pi, \boldsymbol{p}) \text{ if }
    \pi\in \Pi_S\\
    0 \textrm{ otherwise,}
  \end{array}
\right.
\end{equation}
where $\chi:\Pi_S\times [0, 1]^{|\mathcal{V}|}\to \{0, 1\}$ is locally smooth almost everywhere (under Lebesgue measure on $[0,1]$). In other words, for $a.e.$ $\boldsymbol{p}\in [0, 1]^{|\mathcal{V}|}$ and every $\pi\in\Pi_S$ there exists a neighborhood $B_\epsilon(\boldsymbol{p})$ where $\chi(\pi, \cdot)$ is constant. 
\end{hypothesis}
\begin{remark}
Assumption \eqref{singleton} is a special case of hypothesis \ref{general-hypo}, where $\chi(\pi, \boldsymbol{p}) = 1$ if $\pi = \pi_S(\boldsymbol{p})$ and $0$ otherwise.
\end{remark}
\begin{remark}
Without loss of generality, in equations \eqref{singleton} and \eqref{general-hypo-eqn} we replace the symbol ``$\propto$'' with ``$=$'', which otherwise complicates all formulae expressions with a non-essential scalar constant.
\end{remark}
\begin{lemma}
\label{gen-lemma}
Under Hypothesis \ref{general-hypo}, $\boldsymbol{p}_{k+1}$ is a solution to \eqref{em} iff the following holds $\forall\tau^*\in \mathcal{V}$ :
\begin{equation}
\label{solution}
\boldsymbol{p}_{k+1}(\tau^*) = \frac{\sum_{S\in\mathscr{S}}\mathbb{E}_{\Pi_S}[n(\pi, \tau^*) | S; \boldsymbol{p}_k]}{\sum_{\tau\in\mathcal{V}}\sum_{S\in\mathscr{S}}\mathbb{E}_{\Pi_S}[n(\pi, \tau)|S; \boldsymbol{p}_k]}.
\end{equation}
\end{lemma}

\begin{proof}[proof of Theorem \ref{confirmem} assuming Lemma \ref{gen-lemma} holds]
By following the \emph{E-step} of Algorithm \ref{token}, we can express the frequency $\mathcal{F}^*(t)$ as
\begin{equation}
\label{freq}
\sum_{S\in\mathscr{S}}\sum_{\tau\in \pi_S(\boldsymbol{p}_i)}\mathbbm{1}_{\tau=t} = \sum_{S\in\mathscr{S}}n(\pi_S(\boldsymbol{p}_i), t).
\end{equation}
Assumption \ref{singleton} says that the probability measure $\mathbb{P}(\pi|S; \boldsymbol{p}_k)$ is supported on the singleton $\pi_S(\boldsymbol{p}_i)$, therefore the following holds for all $\tau\in\mathcal{V}$:
\begin{equation}
\label{singleton-freq}
n(\pi_S(\boldsymbol{p}_i), \tau) = \mathbb{E}_{\Pi_S}[n(\pi, \tau) | S; \boldsymbol{p}_k].
\end{equation}
Now inserting the identity \eqref{singleton-freq} to the right hand side of equation \eqref{freq} we obtain
\begin{equation}
\label{back}
\mathcal{F}^*(t) = \sum_{S\in\mathscr{S}}\mathbb{E}_{\Pi_S}[n(\pi, t) | S; \boldsymbol{p}_k].
\end{equation}
Next, Inserting \eqref{back} to the \emph{M-step} in Algorithm \ref{token}, we have
\begin{equation}
\label{equiv}
\boldsymbol{p}_{k+1}(t) = \frac{\sum_{S\in\mathscr{S}}\mathbb{E}_{\Pi_S}[n(\pi, t) | S; \boldsymbol{p}_k]}{\sum_{\tau\in\mathcal{V}}\sum_{S\in\mathscr{S}}\mathbb{E}_{\Pi_S}[n(\pi, \tau)|S; \boldsymbol{p}_k]}.
\end{equation}
Invoking Lemma \ref{gen-lemma}, we see that $\boldsymbol{p}_{k+1}$ is the solution to problem \eqref{em}, which proves the first conclusion in Theorem \ref{confirmem}. 

Secondly, the monotonicity of $\log\mathbb{P}(S;\boldsymbol{p}_i)$ can be achieved as follows:
\begin{eqnarray}
\begin{split}
\label{ineq}
&\sum_{S\in \mathscr{S}} \log \mathbb{P}(S;\boldsymbol{p}_{k+1}) \\
=&\sum_{S\in \mathscr{S}} \log \mathbb{P}(S, \pi_S(\boldsymbol{p}_{k+1});\boldsymbol{p}_{k+1})\\
=& \sum_{S\in \mathscr{S}} \boldsymbol{p}_{k+1}(\pi_S( \boldsymbol{p}_{k+1}))\\
\ge& \sum_{S\in \mathscr{S}} \boldsymbol{p}_{k+1}(\pi_S( \boldsymbol{p}_{k}))\\
=& \sum_{S\in \mathscr{S}} \log \mathbb{P}(S, \pi_S(\boldsymbol{p}_{k});\boldsymbol{p}_{k+1})\\
=& \sum_{S\in \mathscr{S}} \mathbb{E}[\log \mathbb{P}(S, \pi_S(\boldsymbol{p}_{k});\boldsymbol{p}_{k+1})| S; \boldsymbol{p}_k]\\
\ge& \sum_{S\in \mathscr{S}} \mathbb{E}[\log \mathbb{P}(S, \pi_S(\boldsymbol{p}_{k});\boldsymbol{p}_{k})| S; \boldsymbol{p}_k]\\
=& \sum_{S\in \mathscr{S}} \boldsymbol{p}_{k}(\pi_S( \boldsymbol{p}_{k}))\\
=&\sum_{S\in \mathscr{S}} \log \mathbb{P}(S, \pi_S(\boldsymbol{p}_{k});\boldsymbol{p}_{k})\\
=&\sum_{S\in \mathscr{S}} \log \mathbb{P}(S;\boldsymbol{p}_{k}).
\end{split}
\end{eqnarray}
Here, all equalities are consequences of Assumption \ref{singleton}, the first inequality follows from the definition of $\pi_S(\boldsymbol{p}_{k+1})$, and the second inequality uses the maximization property of $\boldsymbol{p}_{k+1}$: 
\begin{equation}
\boldsymbol{p}_{k+1}=\text{argmax}_{\boldsymbol{p}}\sum_{S\in \mathscr{S}} \mathbb{E}[\log \mathbb{P}(S, \pi_S(\boldsymbol{p}_{k});\boldsymbol{p})| S; \boldsymbol{p}_k].
\end{equation}
Thus concludes Theorem \ref{confirmem}.
\end{proof}
To complete the proof of Theorem \ref{confirmem} it suffices to prove Lemma \ref{gen-lemma}.
\begin{proof}[proof of Lemma \ref{gen-lemma}]
Consider the following Lagrange multiplier of problem \eqref{em}:
\begin{eqnarray}
\begin{split}\mathcal{L}(\boldsymbol{p}, \lambda) =& \sum_{S\in\mathscr{S}}\mathbb{E}\big{[}\log \mathbb{P}(S, \pi; \boldsymbol{{p}}) \big{\rvert} S; \boldsymbol{{p}_k}\big{]} \\
&+ \lambda(\sum_{\tau\in\mathcal{V}}\boldsymbol{p}(\tau) - 1).
\end{split}
\end{eqnarray}
Plugging \eqref{general-hypo-eqn} into the above equation, we get
\begin{eqnarray}
\begin{split}
\label{terms}
& \sum_{S\in\mathscr{S}}\sum_{\pi\in\Pi_S} \log \boldsymbol{p}(\pi)\cdot \mathbb{P}(\pi\big{\rvert}S; \boldsymbol{p}_k) \\
+ & \sum_{S\in\mathscr{S}}\sum_{\pi\in\Pi_S} \log \chi(\pi, \boldsymbol{p}) \cdot \mathbb{P}(\pi\big{\rvert}S; \boldsymbol{p}_k) \\
+ & \lambda(\sum_{\tau\in\mathcal{V}}\boldsymbol{p}(\tau) - 1) \\
\coloneqq &\text{I} + \text{II} + \text{III}.
\end{split}
\end{eqnarray}
Inserting equation \eqref{terms} to the following identity:
\begin{equation}
\label{lagrange}
\nabla_{\boldsymbol{p}, \lambda} \mathcal{L} = \boldsymbol{0},
\end{equation}
we obtain for each $\tau^*\in\mathcal{V}$ that
\begin{eqnarray}
\begin{split}
\label{crit}
\sum_{S\in\mathscr{S}}\sum_{\pi\in\Pi_S}\frac{n(\pi,\tau^*)}{\boldsymbol{p}(\tau^*)}\cdot \mathbb{P}(\pi\big{\rvert} S; \boldsymbol{p}_k) + \lambda = 0.
\end{split}
\end{eqnarray}
Note the locally constant assumption in Hypothesis \ref{general-hypo} makes the derivative of term II vanishes $a.e.$. Identity \eqref{crit} then allows us to solve for $\boldsymbol{p}(\tau^*)$:
\begin{equation}
\label{eachtp}
\boldsymbol{p}(\tau^*)=-\frac{1}{\lambda}\cdot \sum_{S\in\mathscr{S}}\sum_{\pi\in\Pi_S}n(\pi,\tau^*)\cdot \mathbb{P}(\pi\big{\rvert} S; \boldsymbol{p}_k).
\end{equation}
Next, using the simplex property $\sum_{\tau\in\mathcal{V}} \boldsymbol{p}(\tau) = 1$ and summing up \eqref{eachtp} over $\mathcal{V}$, we find $\lambda$:
\begin{equation}
-\frac{1}{\sum_{\tau\in\mathcal{V}}\sum_{S\in\mathscr{S}}\sum_{\pi\in\Pi_S}n(\pi,\tau)\cdot \mathbb{P}(\pi\big{\rvert} S; \boldsymbol{p}_k)}.
\end{equation}
Plugging the above value of $\lambda$ back to \eqref{eachtp}, we obtain the final expression of $\boldsymbol{p}(\tau^*)$:
\begin{eqnarray}
\begin{split}
\label{conclusion}
&\frac{\sum_{S\in\mathscr{S}}\sum_{\pi\in\Pi_S}n(\pi,\tau^*)\cdot \mathbb{P}(\pi\big{\rvert} S; \boldsymbol{p}_k)}{\sum_{\tau\in\mathcal{V}}\sum_{S\in\mathscr{S}}\sum_{\pi\in\Pi_S}n(\pi,\tau)\cdot \mathbb{P}(\pi\big{\rvert} S; \boldsymbol{p}_k)}\\
=&\frac{\sum_{S\in\mathscr{S}}\mathbb{E}\big{[}n(\pi,\tau^*)\big{\rvert}S; \boldsymbol{p}_k\big{]}}{\sum_{\tau\in\mathcal{V}}\sum_{S\in\mathscr{S}}\mathbb{E}\big{[}n(\pi,\tau)\big{\rvert}S; \boldsymbol{p}_k\big{]}},
\end{split}
\end{eqnarray}
which is precisely \eqref{solution}. This proves the \emph{if} direction of the Lemma. Indeed, a maximizer $\boldsymbol{p}$ must be a critical point of the Lagrange multiplier and satisfy \eqref{lagrange}, therefore identity \eqref{conclusion} holds. Conversely, identity \eqref{conclusion} for arbitrary $\tau^*\in\mathcal{V}$ fully characterize $\boldsymbol{p}$, and by the \emph{if} direction it can only be the unique maximum. This proves the opposite direction, and completes the proof of Lemma \ref{gen-lemma}.
\end{proof}

\section{Appendix}
\label{ffbs}
We propose Algorithm \ref{ffbs-algo}, a Forward-Filtering Backward-Sampling (FFBS) \cite{scott2002,Kudo2018SubwordRI,Kudo2018SentencePieceAS} algorithm under the setting of TreePiece. We highlight those lines in Algorithm \ref{ffbs-algo} that differ from Algorithm \ref{viterbi-algo}. Their main distinctions lie in (1) update formula for probabilities, (2) backward strategy. 

\sethlcolor{lime}
\begin{algorithm}
\caption{Forward-Filtering Backward Sampling (FFBS) Algorithm}\label{ffbs-algo}
\begin{algorithmic}
\Require TreePiece vocabulary $\mathcal{V}$, TreePiece simplex $\boldsymbol{p}$, skeleton $S$, \hl{sampling coefficient $\theta$}.
\Ensure Parition $\pi_S(\boldsymbol{p})$ and probability $\mathbb{P}(S; \boldsymbol{p})$.
\State $\mathcal{T} \gets$ All subtrees of $S$ with the same root. 
\State $\mathscr{L} \gets$ $\texttt{GetInitLogProbs}(\boldsymbol{p})$
\State \hl{$\mathscr{Q} \gets$ $\exp\circ\mathscr{L}$}
\State \hl{$\mathscr{P} \gets$ \texttt{GetInitPairProbs}$(\boldsymbol{p})$}
\State $d_{\text{max}}\gets\text{Depth of {S}}$
\State \textcolor{red}{// Forward begins}
\For{$d=1,2,\cdots,d_{\text{max}}$}
    \For{$\mathfrak{t}\in \mathcal{T}_d$}
        \For{$d^\prime=1,2,\cdots,d$}
            \For{$\mathfrak{t}^\prime\in \texttt{Filter}(\mathcal{T}_{d^\prime}, \mathfrak{t})$}
                \State $\Delta^*\gets$ $\mathfrak{t}^\prime\Delta\mathfrak{t}$ 
            
                \State \hl{$Q^*\gets\mathscr{Q}(\mathfrak{t}^\prime) \cdot \prod_{\tau\in \Delta^*}\boldsymbol{p}(\tau)$}

                \State \hl{$\mathscr{Q}(\mathfrak{t}) \gets \mathscr{Q}(\mathfrak{t}) + Q^*$}
                \State \hl{$\mathscr{P}(\mathfrak{t}, \mathfrak{t}^\prime)\gets Q^*$}
            \EndFor
        \EndFor
    \EndFor
\EndFor
\State \hl{$\mathbb{P}(S; \boldsymbol{p})\gets \mathscr{Q}(S)$}
\State \textcolor{red}{// Forward ends}
\State $\mathfrak{t}_{\text{curr}}\gets S$, $\pi_S(\boldsymbol{p})\gets \emptyset$
\State \textcolor{red}{// Backward begins}
\While{$\mathfrak{t}_{\text{curr}}\neq \texttt{BOS}$ token,}
    \State \hl{$\mathfrak{t}_{\text{prev}}\gets \texttt{Sampling}(\mathscr{P}, \mathfrak{t}_{\text{curr}}, \theta)$}
    \State $\Delta^*\gets\mathfrak{t}_{\text{prev}}\Delta\mathfrak{t}_{\text{curr}}$
    \State $\pi_S(\boldsymbol{p}) \gets \pi_S(\boldsymbol{p})\bigcup \Delta^*$
    \State $\mathfrak{t}_{\text{curr}}\gets \mathfrak{t}_{\text{prev}}$
\EndWhile
\State $\pi_S(\boldsymbol{p}) \gets \pi_S(\boldsymbol{p})\bigcup \{\mathfrak{t}_{\text{curr}}\}$
\State \textcolor{red}{// Backward ends}
\State \Return $\pi_S(\boldsymbol{p}), \mathbb{P}(S; \boldsymbol{p})$
\end{algorithmic}
\end{algorithm}

Before \emph{forward}, we call \texttt{GetInitPairProbs} to initialize a probability function on the Cartesian product space $\mathcal{T}\times \mathcal{T}\bigcup \{\texttt{BOS}\}$ as follows:
\[\mathscr{P}(\mathfrak{t}, \mathfrak{t}^\prime)=\begin{cases} 
      \boldsymbol{p}(\mathfrak{t}) &\text{ if }\mathfrak{t}\in \mathcal{V}\text{ and }\mathfrak{t}^\prime = \texttt{BOS}\text{ token}, \\
      0 &\text{ otherwise}.
   \end{cases}\]
During \emph{backward}, we call \texttt{Sampling} to randomly sample a previous subtree of $\mathfrak{t}$ with respect to the following distribution:
\begin{equation}
\Big{\{} \frac{\exp(\theta\cdot\log\mathscr{P}(\mathfrak{t}^\prime, \mathfrak{t}))}{\sum_{\mathfrak{s}\in \mathcal{T}(\mathfrak{t})} \exp(\theta\cdot\log\mathscr{P}(\mathfrak{s}, \mathfrak{t}))}\Big{\}}_{\mathfrak{t}^\prime \in \mathcal{T}(\mathfrak{t})}
\end{equation}
where $\mathcal{T}(\mathfrak{t}) \eqqcolon \{t^\prime \in \mathcal{T}: \mathscr{P}(\mathfrak{t}, \mathfrak{t}^\prime) > 0\}$. Here a smaller $\theta$ leads to a more uniform sampling distribution among all partitions, while a larger $\theta$ tend to select the Viterbi partition picked by Algorithm \ref{viterbi-algo} \cite{Kudo2018SubwordRI}.

Algorithm \ref{ffbs-algo} allows us to sample from all possible partitions rather than generating fixed patterns. In practice, this version is used in place of Algorithm \ref{viterbi-algo} to reduce the OOV rates; see Appendix \ref{oov} for further discussions.

\begin{remark}
\emph{Let us assume the following holds in place of Assumption \eqref{singleton}:}
\begin{equation}
\label{dense}
   \mathbb{P}(S, \pi; \boldsymbol{p}) \propto \left\{ 
  \begin{array}{ c l }
    \prod_{\tau\in \pi} \boldsymbol{p}(\tau) & \text{if }
    \pi\in \Pi_S\\
    0                 & \textrm{otherwise,}
  \end{array}
\right.
\end{equation}
\emph{another special case of Hypothesis \ref{general-hypo} with $\chi(\pi,\boldsymbol{p})\equiv 1$. By Lemma \ref{gen-lemma}, solving problem \eqref{em} requires computing $\mathbb{E}_{\Pi_S}[n(\pi, \tau)|S; \boldsymbol{p}_k]$, which now becomes NP-hard. But we can utilize Algorithm \ref{ffbs-algo} to obtain an approximate solution. Indeed, if we iteratively run Algorithm \ref{ffbs-algo} in place of the \emph{E-step} in Algorithm \ref{em} $K$ times to obtain a partition sequence $\pi_S(\boldsymbol{p})^{(1)},\pi_S(\boldsymbol{p})^{(2)},\cdots, \pi_S(\boldsymbol{p})^{(K)}$, and use the averaged partitions to update the frequency $\mathcal{F}^*$, then following similar lines in Appendix \ref{proof}, we can prove an asymptotic version of Theorem \ref{confirmem} under Assumption \ref{dense}, by showing that the averaged frequency over $K$ partitions converges to $\mathbb{E}\big{[}n(\pi,\tau)\big{\rvert}S; \boldsymbol{p}_k\big{]}$ as $K$ tends to infinity, a direct consequence of \emph{Law of Large Numbers}. We omit the details.}
\end{remark}

\begin{table*}[]\centering
\resizebox{1.0\textwidth}{!}{%
\begin{tabular}{c|ccc|ccc|cl}
\hline
                                                              & \multicolumn{3}{c|}{TreePiece model} & \multicolumn{3}{c|}{NAR baseline} & \multicolumn{2}{c}{AR baseline} \\ \hline
Modules &
  Encoder &
  \begin{tabular}[c]{@{}c@{}}TreePiece \\ decoder\end{tabular} &
  \begin{tabular}[c]{@{}c@{}}Utterance \\ decoder\end{tabular} &
  Encoder &
  \begin{tabular}[c]{@{}c@{}}Length \\ predictor\end{tabular} &
  Decoder &
  Encoder &
  Decoder \\ \hline
\begin{tabular}[c]{@{}c@{}}Learning \\ rates\end{tabular} &
  $4\times 10^{-6}$ &
  $6\times 10^{-5}$ &
  $4\times 10^{-5}$ &
  $2\times 10^{-5}$ &
  $1\times 10^{-4}$ &
  $6\times 10^{-5}$ &
  $4\times 10^{-6}$ &
  $2\times 10^{-5}$ \\ \hline
\begin{tabular}[c]{@{}c@{}}Decay \\ coefficients\end{tabular} & \multicolumn{3}{c|}{$0.999$}         & \multicolumn{3}{c|}{$0.999$}      & \multicolumn{2}{c}{$0.9999$}    \\ \hline
\end{tabular}%
}
\caption{Optimization hyperparameter choices for all models}
\label{hp-choices}
\end{table*}
\begin{figure*}[t]
     \centering
     \begin{minipage}[t]{1.0\linewidth}
         \centering
         \includegraphics[width=0.8\textwidth, height=2cm]{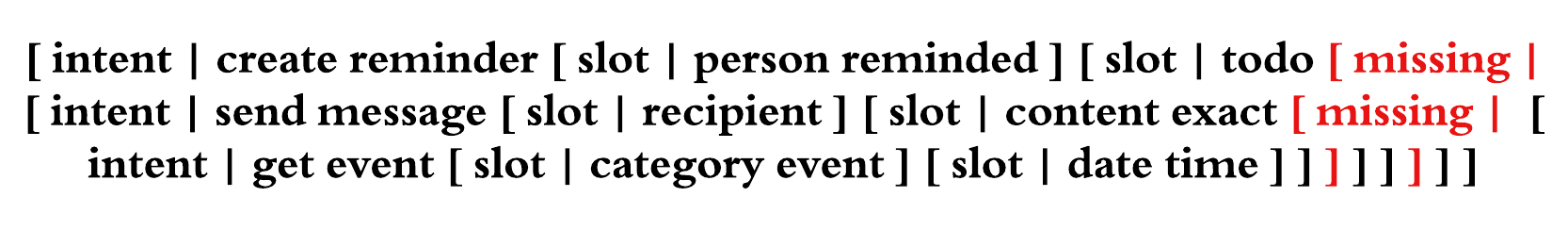}
     \end{minipage}
     \caption{Serializing the skeleton in Figure \ref{glue} (LHS) using the \emph{``placeholder nest''} design.}
     \label{ser}
\end{figure*}

\section{Appendix}
\label{oov}
As discussed in Subsection \ref{recover}, a placeholder structure is necessary for well-defined assembly of \emph{TreePiece units}. However, adding all possible placeholder patterns to vocabulary is impractical for both time and memory. Instead, we shall include only those patterns that most likely to occur. To do so, we tokenize every training skeleton and add the results to the \emph{TreePiece vocabulary}. As illustrated by the ``Tokenize'' direction in Figure \ref{glue}, when a node loses a child during tokenization, we attach a placeholder to the missing child's position. 
\begin{remark}\label{oov-sampling}
\emph{There may exist new placeholder patterns that are \emph{Out-Of-Vocabulary} (OOV) at inference time. To mitigate OOV, we apply Algorithm \ref{ffbs-algo} (in place of Algorithm \ref{viterbi-algo}) to tokenize each training skeleton $K_0$ times. Both $K_0$ and the sampling coefficient $\theta_0$ will be specified in Appendix \ref{tp-vocab-hps}. Intuitively, with a larger $K_0$ and a smaller $\theta_0$, Algorithm \ref{ffbs-algo} is able to generate more abundant placeholder patterns to cover as many OOV placeholders as possible.}
\end{remark}


\section{Appendix}
\label{conf-hps}
\subsection{Model configurations}
\subsubsection{Model architectures}
Across all of our experiments, we use RoBERTa-base encoder \cite{Liu2019RoBERTaAR} and transformer \cite{Vaswani2017AttentionIA} decoders. For encoder architecture, we refer the readers to \cite{Liu2019RoBERTaAR}. All models' decoder have the same multi-head-transformer-layer architecture (embedding dimension $768$, $12$ heads). Both AR and NAR's decoders have $2$ layers. For TreePiece archtecture, \emph{TreePiece decoder} and \emph{Utterance decoder} has $1$ layer each. Note for fairness of comparisons, we let each model have exactly $2$ decoder layers in total.

\subsubsection{TreePiece vocabulary}
\label{tp-vocab-hps}
We extract from the TOPv2 dataset $162$ ontologies in total, and use TOPv2 training split as training corpus to iteratively run vocabulary generation (ref. Subsection \ref{vocab-gen}) $600$ times to obtain a vocabulary of size $762$. Then apply the Algorithm \ref{token} with $N_0 = 30$ and $\epsilon_0 = 0.01$ to train the tokenizer. Finally, we follow Appendix \ref{oov} (with $K_0= 10, \theta_0 = 0.15$) and expand the \emph{TreePiece vocabulary} to size $2153$. Note the vocabulary obtained this way has less than $0.1\%$ OOV rate on test dataset, compared to $0.45\%$ were we not using the sampling trick in Remark \ref{oov-sampling}.

\subsection{Training details}
\subsubsection{TreePiece embedding}
Within the \emph{TreePiece decoder}, we tie the classifer head's weight to the \emph{TreePiece unit} embedding matrix, and found it beneficial to pretrain this weight rather than randomly initializing it. We  take inspirations from \cite{Shrivastava2022RetrieveandFillFS} and create the pretraining corpus by serializing all \emph{skeletons} in the training dataset. To let the placeholder information blend into the corpus, we introduce a \emph{placeholder nest} structure and add it to the serialized logical forms, as illustrated by Figure \ref{ser}. Finally, we use the masked language model (MLM) \cite{Devlin2019BERTPO} pre-training objective with mask-rate $0.15$ and train for up to $20$ epochs until convergence.

\subsubsection{Hyperparameter choices}
\label{hps}
Across all experiments, we set the batch size to be $256$, and total number of epochs to be $100$ with early stopping when validation EM (ref. Subsection \ref{metrics}) stops improving. For optimization, we use Adam optimizer \cite{Kingma2014AdamAM} with weight decay $0.01$ and $\epsilon=10^{-8}$. In addition, we warm up the learning rate in $5$ epochs and then exponentially decay \cite{Senior2013AnES} at the end of every epoch. 

We also observe that each module favors learning rates with different magnitude, so we do learning rate search separately for each module among the interval $[10^{-6}, 10^{-4}]$. For exponential decay coefficients we optimize them among $\{0.9, 0.99, 0.999, 0.9999\}$. Table \ref{hp-choices} summarizes our final choices of these hyperparameters.

\subsection{Comparisons with prior work}
\label{comparisons}
In order to prioritize the main thrust of our paper, we opted for standard and straightforward training procedures in our experiments, without incorporating additional techniques such as label smoothing, R3F loss, beam search decoding, etc. As a result, the baseline numbers reported in Table \ref{key-results} cannot be directly compared to those of previous works \cite{Rongali2020DontPG,Shrivastava2021SpanPN,Shrivastava2022RetrieveandFillFS} that heavily rely on these training methodologies.

\end{document}